%% file: preprint.tex
\documentclass{article}

\usepackage{arxiv}

\usepackage[utf8]{inputenc} 
\usepackage[T1]{fontenc}    
\usepackage{hyperref}       
\usepackage{url}            
\usepackage{booktabs}       
\usepackage{amsfonts}       
\usepackage{nicefrac}       
\usepackage{microtype}      
\usepackage{lipsum}		
\usepackage{graphicx}
\usepackage{natbib}
\usepackage{doi}
\usepackage{mathtools}
\usepackage{amsmath}
\usepackage{amssymb}
\usepackage{amsthm}
\usepackage{bm}

\theoremstyle{definition}
\newtheorem{definition}{Definition}[section]
\newtheorem{theorem}{Theorem}[section]

\newtheorem{proposition}[theorem]{Proposition}
\newtheorem{example}[theorem]{Example}
\newtheorem{remark}[theorem]{Remark}

\newcommand{\argmax}{\mathop{\rm arg~max}\limits}
\newcommand{\argmin}{\mathop{\rm arg~min}\limits}

\title{Information Geometry of Dropout Training}

\author{ \href{https://orcid.org/0000-0002-9953-3469}{\includegraphics[scale=0.06]{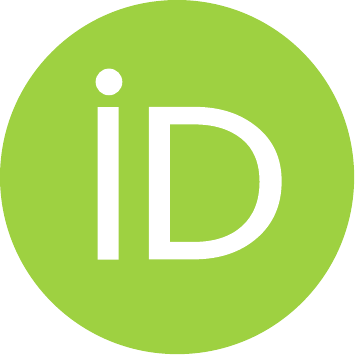}\hspace{1mm}Masanari Kimura} \\
    Department of Statistical Science\\
	School of Multidisciplinary Sciences\\
	The Graduate University for Advanced Studies (SOKENDAI)\\
	Kanagawa 240-0193, Japan \\
	\texttt{mkimura@ism.ac.jp} \\
	\And
	\href{https://orcid.org/0000-0002-6405-4361}{\includegraphics[scale=0.06]{orcid.pdf}\hspace{1mm}Hideitsu Hino} \\
	The Institute of Statistical Mathematics\\
	Tokyo 190-0014, Japan \\
	\texttt{hino@ism.ac.jp} \\
}

\renewcommand{\shorttitle}{Information Geometry of Dropout Training}

\hypersetup{
pdftitle={Information Geometry of Dropout Training},
pdfauthor={Masanari Kimura, Hideitsu Hino},
pdfkeywords={information geometry, machine learning, dropout},
}

\begin{document}
\maketitle

\begin{abstract}
Dropout is one of the most popular regularization techniques in neural network training.
Because of its power and simplicity of idea, dropout has been analyzed extensively and many variants have been proposed.
In this paper, several properties of dropout are discussed in a unified manner from the viewpoint of information geometry.
We showed that dropout flattens the model manifold and that their regularization performance depends on the amount of the curvature.
Then, we showed that dropout essentially corresponds to a regularization that depends on the Fisher information, and support this result from numerical experiments.
Such a theoretical analysis of the technique from a different perspective is expected to greatly assist in the understanding of neural networks, which are still in their infancy.
\end{abstract}

\section{Introduction}
Deep neural networks have been experimentally successful in a variety of fields~\citep{deng2014deep,lecun2015deep,goodfellow2016deep}.
Dropout is one of the techniques that contribute to the performance improvement of neural networks~\citep{srivastava2014dropout}.
Many experimental results have reported the effectiveness of dropout, making it an important technique for training neural networks~\citep{wu2015towards,pham2014dropout,park2016analysis,labach2019survey}.
Furthermore, the simplicity of the idea of dropout has led to the proposal of a great number of variants~\citep{iosifidis2015dropelm,moon2015rnndrop,gal2017concrete,zolna2017fraternal,hou2019weighted,keshari2019guided,ma2020dropout}.
Understanding the behavior of such an important technique can be a way to know which of these variants to use, and in what cases dropout is effective in the first place.
Several empirical and theoretical analyses of dropout have been conducted~\citep{warde2013empirical,baldi2013understanding,gal2016dropout,zhao2019equivalence,garbin2020dropout,nalisnick2019dropout}.

This paper aims to clarify the behavior of dropout using information geometry~\citep{amari2016information,ay2017information}.
With this new perspective, we reveal the following properties of dropout.
First, we showed that dropout flattens the model manifold and that their regularization performance depends on the amount of the curvature.
A similar phenomenon has been investigated for Bayesian predictive models and bootstrap methods~\citep{komaki1996asymptotic,fushiki2004parametric}.
Finally, we show that dropout essentially corresponds to a regularization that depends on the Fisher information, and support this result from numerical experiments.
This result suggests that dropout and other regularization methods can be generalized via the Fisher information matrix.

\section{Preliminaries}

\subsection{Empirical risk minimization}
Let $\mathcal{X}\in\mathbb{R}^d$ be the $d$-dimensional input space and $\mathcal{Y}$ be the output space, where $d\in\mathbb{N}$.
In this paper, $\mathcal{Y}=\{1,\dots,K\}$ for $K\in\mathbb{N}$ is assumed, which corresponds to a $K$-class classification problem.

Let $\mathcal{H}$ be a hypothesis class. 
The goal of supervised learning is to obtain a hypothesis $h:\mathcal{X}\to \mathbb{Y}\ (h\in\mathcal{H})$ with the  training set $D = \{(\bm{x}_i, y_i)\}^N_{i=1}$ of sample size $N\in\mathbb{N}$ that minimizes the expected loss:
\begin{equation}
    \mathcal{R}(h) \coloneqq \mathbb{E}_{(\bm{x},y)\sim P}\Big[\ell(h(\bm{x}), y)\Big], \label{eq:expected_loss}
\end{equation}
where $P$ is the unknown distribution and $\ell: \mathbb{R}\times\mathcal{Y}\to\mathbb{R}_+$ is the loss function that measures the discrepancy between the true output value $y$ and the predicted value $\hat{y}\coloneqq h(\bm{x})$.
In general, access to the unknown distribution $P$ is forbidden.
If the i.i.d. assumption holds, i.e., $D$ is generated independently and identically from $P$, then the following empirical risk minimization (ERM) can approximate the expected risk minimization problem~\citep{vapnik1999overview,vapnik2013nature}:
\begin{equation}
\mbox{minimize} \; \hat{\mathcal{R}}(h), \quad 
    \hat{\mathcal{R}}(h) \coloneqq \frac{1}{N}\sum^N_{i=1}\ell(h(\bm{x}_i), y_i)).
\end{equation}
ERM framework is very powerful and known to be consistent:
\begin{equation}
    \mathbb{E}_{(\bm{x},y)\sim P}\Big[\hat{\mathcal{R}}(h)\Big] = \frac{1}{N}\sum^N_{i=1}\mathbb{E}_{(\bm{x},y)\sim P}\Big[\ell(h(\bm{x}), y)\Big] = \mathbb{E}\Big[\ell(h(\bm{x}), y)\Big] = \mathcal{R}(h).
\end{equation}
However, ERM often suffers from the problem of overtraining due to insufficient sample size of training data or too complex hypothesis classes (e.g. neural networks).

\subsection{Dropout training}
Especially when the hypothesis class is a set of neural networks, the dropout~\citep{srivastava2014dropout} is a frequently used technique to prevent overtraining.
Dropout training aims to improve the generalization performance by repeating the procedure of randomly deleting neurons during the training of neural networks.
Due to the simplicity of the idea and the high effectiveness witnessed from numerical experiments, many variants of dropout have been proposed~\citep{kingma2015variational,keshari2019guided,hu2020surrogate,liu2019beta,chen2020dropcluster,ke2020group}.
In addition, a number of theoretical and experimental analyses of dropout have been reported~\citep{baldi2013understanding,helmbold2015inductive,warde2013empirical}.
In this paper, dropout is once again analyzed through the lens of information geometry.

\subsection{Information geometry of machine learning}
It is known that machine learning procedures can be formulated on Riemannian manifolds~\citep{amari1995information,amari2000methods,amari2016information,ay2017information}.
\begin{definition}
Let $\mathcal{M}$ be a metric space and $U\subset\mathcal{M}$ be an open set. Then the pair $(U, \phi)$ is called a coordinate system or a chart on $\mathcal{M}$, if $\phi:U\to\phi(U)\subset\mathbb{R}^n$ is a homeomorphism of the open set $U$ in $\mathcal{M}$ onto an open set $\phi(U)$ of $\phi(P) = (x^1(P),\dots,x^n(P))$, $P \in U$, namely $x^j=u^j\circ\phi$, where $u^j:\mathbb{R}^n\to\mathbb{R}$, $u^j(a_1,\dots,a_n)=a_j$ is the $j$-th projection.
\end{definition}
The integer $n \in\mathbb{N}$ is the dimension of the coordinate system or parameter space.

\begin{definition}
An atlas $\mathcal{A}$ of dimension $n\in\mathbb{N}$ associated with the metric space $\mathcal{M}$ is a collection of coordinate systems $\{(U_\alpha, \phi_\alpha)\}_\alpha$ such that
\begin{enumerate}
    \item $\forall{\alpha}, U_\alpha\subset\mathcal{M},\ 
\bigcup_\alpha U_\alpha = \mathcal{M}$,
    \item if $U_\alpha\cap U_\beta\neq\emptyset$, the restriction to $\phi_\alpha(U_\alpha\cap U_\beta)$ of the map
    \begin{equation*}
        F_{\alpha\beta} = \phi_\beta\circ\phi^{-1}_\alpha: \phi_\alpha(U_\alpha\cap U_\beta) \to \phi_\beta(U_\alpha\cap U_\beta)
    \end{equation*}
    is differentiable from $\mathbb{R}^n$ to $\mathbb{R}^n$.
\end{enumerate}
\end{definition}

\begin{definition}
A differentiable manifold $\mathcal{M}$ is a metric space endowed with a complete atlas.
The dimension $n$ of the atlas is called the dimension of the manifold.
\end{definition}

From the definition, a differentiable manifold is required to satisfy that the Jacobian of the map $\phi_\beta\circ\phi_\alpha^{-1}$ has a maximum rank.
A hypothesis set $\mathcal{H}$ or a statistical model $\mathcal{M}$ in machine learning can be treated as a manifold.
\begin{example}[Manifold of one-layer neural networks:]
As an example, we now consider a one-layer neural network with sigmoid function $\sigma(z)=1/(1+e^{-z})$.
Let $\bm{x}\in\mathbb{R}^n$ be an $n$-dimensional input and $y=\sigma(\bm{\theta}^T\bm{x} + \theta_0)$ be the one-dimensional output, where $\bm{\theta}\in\mathbb{R}^n$ and $\theta_0\in\mathbb{R}$ are the weights and the bias.
Then the set of outputs
\begin{align*}
    \mathcal{H}_\sigma = \{h(\bm{x}; \bm{\theta}, \theta_0)\ |\ \bm{\theta}\in\mathbb{R}^n, \theta_0 \in\mathbb{R}\} = \{\sigma(\bm{\theta}^T\bm{x} + \theta_0)\ | \bm{\theta}\in\mathbb{R}^n, \theta_0\in\mathbb{R}\}
\end{align*}
can be regarded as an $(n+1)$-dimensional manifold, parameterized by $\bm{\theta}$ and $\theta_0$.

In practice, we can verify that the Jacobian matrix of $y=y(\bm{\theta}, \theta_0)$ are full-rank:
From the properties of the sigmoid function, we have
\begin{align*}
    \frac{\partial y}{\partial \theta_0} &= \sigma'(\bm{\theta}^T\bm{x} + \theta_0) = y(1-y), \\
    \frac{\partial y}{\partial \theta_j} &= \sigma'(\bm{\theta}^T\bm{x} + \theta_0)x_j = y(1-y)x_j.
\end{align*}
Then, for $a_0, a_1,\dots,a_n\in\mathbb{R}$, let $a_0\frac{\partial y}{\partial \theta_0} + \sum^n_{j=1}a_j\frac{\partial y}{\partial \theta_j} = 0$.
Since $y(1-y)\neq 0$ for finite $\bm{\theta}^\top\bm{x} + \theta_0$, the previous relation becomes $a_0 + \sum^n_{j=1}a_jx_j = 0$.
Since this relation holds for any $x_j\in\mathbb{R}$, it follows that $a_0=a_1=\cdots=a_n=0$.
Therefore, $\{\frac{\partial y}{\partial \theta_0}, \frac{\partial y}{\partial \theta_1},\dots,\frac{\partial y}{\partial \theta_n}\}$ are linearly independent, hence the rank of Jacobian of $y$ is $n+1$. 
\end{example}

Let $\Theta$ be a parameter space, and $\bm{\theta}\in\Theta$ be a parameter of some hypothesis $h(\bm{x};\bm{\theta})$.
Suppose that the hypothesis class $\mathcal{H}_{\bm{\theta}}$ has the following structure:
\begin{align*}
    \mathcal{H}_{\bm{\theta}} \coloneqq \Big\{h(\bm{x};\bm{\theta}) = \argmax_{y\in\mathcal{Y}} p(y|\bm{x}; \bm{\theta}) \mid p\in\mathcal{M}, \bm{\theta}\in\Theta \Big\}.
\end{align*}

Assume now that a hypothesis $h\in\mathcal{H}_{\bm{\theta}}$ is trained to approximate the target distribution $q$.
In general, target distributions 
are not included in $\mathcal{M}$.
In this case, we need to find a distribution $p^* = \argmin_{p \in \mathcal{M}} \ D(q, p)$ for some divergence $D$, which corresponds to the orthogonal projection of $q$ on the surface of $\mathcal{M}$.
Since we cannot observe the true target distribution $q$, we use the empirical distribution $q^{emp}$ to approximate $q$.
Here, we assume that $p$, $q$ and $q^{emp}$ are parameterized by $\bm{\theta}_p$, $\bm{\theta}_q$ and $\bm{\eta}_q$.

\section{Information geometry of dropout training}
In this section, dropout procedures are described by information geometry.
Let $\mathcal{M} \coloneqq\{p(y|\bm{x};\bm{\theta})\ |\ \bm{\theta}\in\Theta \}$ be a class of neural networks and $\mathcal{P}$ be the space of probability measures.
ERM can be regarded as finding the projection from the empirical distribution $q^{emp}(y|\bm{x};\bm{\eta}_q)$ to $\mathcal{M}$:
\begin{equation}
    \bm{\theta}^* = \argmin_{\bm{\theta}\in\Theta}D[q^{emp}\|p_{\bm{\theta}}],
\end{equation}
where $p_{\bm{\theta}} = p(y|\bm{x};\bm{\theta})$, and $D[\cdot\|\cdot] : \mathcal{P}\times\mathcal{P}\to\mathbb{R}_+$ is some divergence.
Here dropout can be viewed as considering multiple projections onto submanifolds $\{\mathcal{M}^D_k\}^K_{k=1}$ with some of the parameters $\bm{\theta}$ set to $0$ (see Fig.~\ref{fig:dropout_submanifolds}).
\begin{equation}
    \mathcal{M}^D_k = \{p(y|\bm{x};\bm{\theta}_k)\ |\ \bm{\theta}_k \in\Theta_k \subset\Theta \},\quad 
\end{equation}
where $\Theta_k\subset\Theta$ such that for $I_k \subset \{1,2,\dots,d\}$, $\bm{\theta}_{kj}=0\ (\forall j\in I_k)$ and $\{I_k\}^K_{k=1}$ is a set of possible index sets, and $K$ is the number of dropout patterns.
Here, the training of a neural network applying dropout can be regarded as obtaining the weighted average of the projections from the empirical distribution $q^{emp}(y|\bm{x};\bm{\eta}_q)$ to $\{\mathcal{M}^D_k\}^K_{k=1}$.
\begin{align}
    \bm{\theta}^*_D = \sum^K_{k=1}w_k\bm{\theta}^*_k,\quad
    \bm{\theta}^*_k = \argmin_{\bm{\theta}_k\in\Theta_k}D[q^{emp}\|p_{\bm{\theta}_k}],
\end{align}
where $\bm{w}=\{w_1,\dots,w_K\}, \; \sum_{k=1}^{K}w_k =1$.
Note that whether or not the mixture model can be written as a weighted arithmetic mean is dependent on the choice of divergences and the flatness structure of the model manifold.
In an ordinary dropout, the weights $\bm{w}$ are all identical and $w_k = \frac{1}{K}, k=1,\dots,K$. 

\begin{figure}[t]
    \centering
    \includegraphics[width=0.7\linewidth]{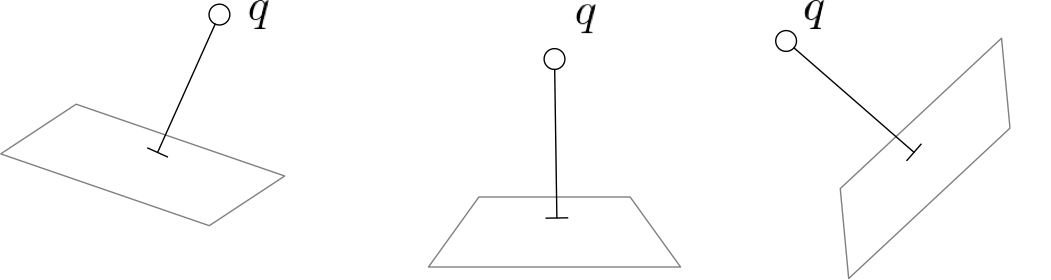}
    \caption{projections from the empirical distribution $q(y|\bm{x};\bm{\eta}_q)$ to $\{\mathcal{M}^D_k\}^\infty_{k=1}$}
    \label{fig:dropout_submanifolds}
\end{figure}

\subsection{Flatness of dropout submanifolds}
\label{sec:flatness}

Although $\bm{\theta}^*_D$ is a weighted average of the parameters on the submanifolds $\{\mathcal{M}^d_{k}\}^K_{k=1}$, this parameter is not necessarily contained in the original model manifold $\mathcal{M}$.
In what follows, the condition of $\bm{\theta}^*_D\in\mathcal{M}$ is confirmed.

\begin{definition}
A Riemannian metric $g$ on a differentiable manifold $\mathcal{M}$ is a symmetric, positive definite $2$-covariant tensor field.
A Riemannian manifold is a differentiable manifold $\mathcal{M}$ endowed with a Riemannian metric $g$.
\end{definition}
\begin{definition}
A function $f:\mathcal{M}\to\mathbb{R}$ is said to be differentiable if for any chart $(U, \phi)$ on $\mathcal{M}$ the function $f\circ\phi^{-1}:\phi(U)\to\mathbb{R}$ is differentiable. The set of all differentiable functions on the manifold $\mathcal{M}$ is denoted by $\mathcal{F}(\mathcal{M})$.
\end{definition}
\begin{definition}
A vector field $X$ on $\mathcal{M}$ is a smooth map $X$ that assigns to each point $p\in\mathcal{M}$ a vector $X_p$ in $T_p\mathcal{M}$.
For any function $f\in\mathcal{F}(\mathcal{M})$ we define the real-valued function $(Xf)_p = X_p f$.
By smooth we mean the following: for each $f\in\mathcal{F}(\mathcal{M})$,  $Xf\in\mathcal{F}(\mathcal{M})$.
The set of all vector fields on $\mathcal{M}$ is denoted by $\mathcal{X}(\mathcal{M})$.
\end{definition}
\begin{definition}
A function $T$ which satisfies
\begin{align}
    T(af + bg) = aT(f) bT(g), \quad \forall{a,b\in\mathbb{R}}, \forall{f,g\in\mathcal{F}(\mathcal{M})}
\end{align}
is said to be $\mathbb{R}$-linear.
\end{definition}
\begin{definition}
A function $T$ which satisfies
\begin{align}
    T(f_1X_1, f_2X_2) &= f_1f_2T(X_1,X_2), \\
    T(X_1 + Y_1, X_2) &= T(X_1, X_2) + T(Y_1, X_2), \\
    T(X_1, X_2 + Y_2) &= T(X_1, X_2) + T(X_1, Y_2)
\end{align}
for any $f_1,f_2\in\mathcal{F}(\mathcal{M})$ and $X_1,X_2,Y_1,Y_2\in\mathcal{X}(\mathcal{M})$ is said to be $\mathcal{F}(\mathcal{M})$-linear.
\end{definition}

The differential of a function $f\in\mathcal{F}(\mathcal{M})$ is defined at any point $p$ by $(df)_p:T_p\mathcal{M}\to\mathbb{R}$,
\begin{align}
    (df)_p(v) = v(f), \quad \forall v\in T_p\mathcal{M}.
\end{align}
In local coordinates $(\theta^1,\dots,\theta^n)$ this takes the form $df=\sum_i \frac{\partial f}{\partial \theta^i}d\theta^i$, where $\{d\theta^i\}$ is the dual basis of $\{\partial / \partial\theta^i\}$ of $T_p\mathcal{M}$, i.e.,
\begin{align}
    d\theta^i\left(\frac{\partial}{\partial\theta^j}\right) = \delta^i_j,
\end{align}
where $\delta^i_j$ denotes the Kronercker symbol.
The space spanned by $\{d\theta^1,\dots,d\theta^n\}$ is called the cotangent space of $\mathcal{M}$ at $p$, and is denoted by $T^*_p\mathcal{M}$.
The elements of $T^*_p\mathcal{M}$ are called covectors.
The differential $df$ is an example of $1$-form.
In general, a $1$-form $\omega$ on the manifold $\mathcal{M}$ is a map which assigns an element $\omega_p\in T^*_p\mathcal{M}$ to each point $p\in\mathcal{M}$.
A $1$-form can be written in local coordinates as
\begin{align}
    \omega = \sum^n_{i=1}\omega_i d\theta^i,
\end{align}
where $\omega_i = \omega(\frac{\partial}{\partial\theta^i})$ is the $i$-th coordinate of the form with respect to the basis $\{d\theta^i\}$.

Riemannian manifold will be denoted from now on by the pair $(\mathcal{M}, g)$.
The Riemannian metric $g$ can be considered as a positive definite scalar product $g_p:T_p\mathcal{M}\times T_p\mathcal{M}\to\mathbb{R}$ that depends on the point $p\in\mathcal{M}$, where $T_p\mathcal{M}$ is a tangent space of $\mathcal{M}$ at $p$.
In local coordinates it can be written as $g = g_{ij}d\theta^id\theta^j$ with $g_{ij}=g_{ji}=g(\partial_i, \partial_j)$.
Here, and hereafter, the Einstein summation convention will be assumed, so that summation will be automatically taken over indices repeated twice in the term, e.g., $\bm{a}^i\bm{b}_i = \sum_{i} \bm{a}^i\bm{b}_i$.
It is known that a natural metric $g$ on a statistical manifold is the Fisher information matrix~\citep{amari2000methods}.

\begin{definition}
A linear connection $\nabla$ on a differentiable manifold $\mathcal{M}$ is a map $\nabla:\mathcal{X}(\mathcal{M})\times\mathcal{X}(\mathcal{M})\to\mathcal{X}(\mathcal{M})$ with the following properties:
\begin{enumerate}
    \item $\nabla_XY$ is $\mathcal{F}(\mathcal{M})$-linear in $\mathcal{X}$;
    \item $\nabla_XY$ is $\mathbb{R}$-linear in $Y$;
    \item It satisfies the Leibniz rule:
    \begin{equation}
        \nabla_X(fY) = (Xf)Y + f\nabla_XY,\quad \forall f\in\mathcal{F}(\mathcal{M}),
    \end{equation}
\end{enumerate}
where $\mathcal{X}(\mathcal{M})$ is the set of vector fields on $\mathcal{M}$ and $\mathcal{F}(\mathcal{M})$ is the set of all differentiable functions on $\mathcal{M}$.
\end{definition}
In a local coordinates system $(\theta^1,\dots,\theta^n)$, one can write
\begin{equation}
    \nabla_{\partial_i}\partial_j = \Gamma^k_{ij}\partial_k,
\end{equation}
where $\Gamma^k_{ij}$ are the coordinates of the connection with respect to the local base $\{\partial_i\}$, where $\partial_i = \frac{\partial}{\partial\theta^i}$.

The followings are definitions that are important for discussing the geometric location between $\mathcal{M}$ and the probability distribution generated by dropout.
\begin{definition}
\label{def:torsion}
Let $\nabla$ be a linear connection.
The torsion tensor is defined as
\begin{align}
    T:\mathcal{X}(\mathcal{M})\times\mathcal{X}(\mathcal{M})\to\mathcal{X}(\mathcal{M}) \nonumber \\
    T(X,Y) = \nabla_XY - \nabla_YX - [X,Y],
\end{align}
where $[\cdot,\cdot]:\mathcal{X}(\mathcal{M})\times\mathcal{X}(\mathcal{M})\to\mathcal{X}(\mathcal{M})$ is the Lie bracket:
\begin{align}
    [X,Y]_pf &= X_p(Yf) - Y_p(Xf), \quad \forall f\in\mathcal{F}, p\in\mathcal{M}, \\
    [X,Y] &= \Big( (\partial_j Y^i ) X^j - (\partial_j X^i) Y^j\Big) \partial_i.
\end{align}
In the coordinate representation, with $X= \partial_i, Y= \partial_j$ and $\gamma^{k}_{ij} \partial_k = [\partial_i,\partial_j]$, the torsion tensor is given by
\begin{equation}
    T^{k}_{ij} = \Gamma^{k}_{ij} - \Gamma^{k}_{ji} - \gamma^{k}_{ij}.
\end{equation}
\end{definition}
\begin{definition}
\label{def:curvature}
Let $\nabla$ be a linear connection. The curvature tensor is defined as
\begin{align}
    R:\mathcal{X}(\mathcal{M})\times\mathcal{X}(\mathcal{M})\times\mathcal{X}(\mathcal{M})\to\mathcal{X}(\mathcal{M}) \nonumber \\
    R(X, Y, Z) = \nabla_X\nabla_YZ - \nabla_Y\nabla_XZ - \nabla_{[X,Y]}Z.
\end{align}
\end{definition}
\begin{definition}
A connection $\nabla$ is called flat in a given system of coordinates if its components vanish: $\Gamma^k_{ij}=0$.
\end{definition}
\begin{proposition}
The torsion and the curvature of a flat connection are zero.
\end{proposition}
\begin{proof}
From Def~\ref{def:torsion} and \ref{def:curvature},
\begin{align}
    T^k_{ij} &= \Gamma^k_{ij} - \Gamma^k_{ji} \\
    R^r_{ijk} &= \partial_i\Gamma^r_{jk} - \partial_j\Gamma^r_{ik} + \Gamma^r_{ih}\Gamma^h_{jk} - \Gamma^r_{jh}\Gamma^h_{ik}.
\end{align}
Substituting $\Gamma^k_{ij} = 0 \ (\forall i,j,k)$ into these gives the proof of the proposition.
\end{proof}

\begin{figure}[t]
    \centering
    \includegraphics[width=0.7\linewidth]{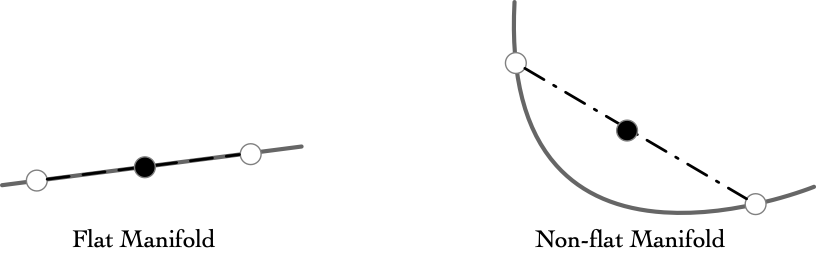}
    \caption{Flat and non-flat manifolds.Whether or not the weighted average between two points is contained in the original manifold is characterized by the flatness of the manifold.}
    \label{fig:flat_vs_nonflat}
\end{figure}

From the above, one can see that a flat manifold is neither curved nor torsional.
Figure~\ref{fig:flat_vs_nonflat} shows the weighted average between two points on a flat manifold and a non-flat manifold.
From the above, $\bm{\theta}^*_D$ is included in $\mathcal{M}$ if the model manifold is flat in the parameter coordinate system.

\subsection{$\alpha$-integration for the dropout training}
\label{subsec:alpha_integration}
In general, it is known that the way parameters are represented in statistical models is not unique.
For example, Gaussian distributions, with natural parameters $(\mu, \sigma^2)$ and polar coordinates representation are equivalent but coordinate repesentation is different.
Now, the neural network is parameterized by $\bm{\theta}$, and the geometry was discussed by considering this $\bm{\theta}$ as a coordinate system ($\bm{\theta}$-coordinate system).
Since the bijections of $\bm{\theta}$ also identify $p\in\mathcal{M}$, other coordinate systems can be considered .
\begin{definition}
Let $(\mathcal{M}, g)$ be a Riemannian manifold.
Two linear connections $\nabla$ and $\nabla^*$ on $\mathcal{M}$ are called dual, with respect to the metric $g$, if
\begin{equation}
    Zg(X,Y) = g(\nabla_ZX, Y) + g(X, \nabla^*_ZY), \quad \forall X,Y,Z\in\mathcal{X}(\mathcal{M}).
\end{equation}
\end{definition}
In local coordinates, 
the duality condition can be expressed as
\begin{equation}
    \partial_k g_{ij} = \Gamma_{ki,j} + \Gamma^*_{kj,i}, \label{eq:local_duality}
\end{equation}
where $\Gamma_{ki,j}=g_{jm}\Gamma^m_{ki}$ and $\Gamma^*_{kj,i}=g_{im}\Gamma^{*m}_{kj}$ are two coordinate components of connections $\nabla$ and $\nabla^*$.
The triplet $(g, \nabla, \nabla^*)$ is called a dualistic structure on $\mathcal{M}$, and quadruple $(\mathcal{M}, g, \nabla, \nabla^*)$ is called a Riemannian manifold endowed with a dualistic structure.
\begin{proposition}
Given a linear connection $\nabla$ on the Riemannian manifold $(\mathcal{M}, g)$, there is a unique connection $\nabla^*$ dual to $\nabla$.
\end{proposition}
\begin{proof}
It suffices to prove the property locally, in a coordinate chart.
For existence, since the connection $\nabla$ is given, its components $\Gamma_{ij,l}$ are known.
Let $\Gamma^*_{ij,l}=\partial_i g_{ij}-\Gamma_{il,j}$ and $\Gamma^{*k}_{ij}=\Gamma^*_{ij,l}g^{lk}$, and construct the dual connection by $\nabla^*_{\partial_i}\partial_j=\Gamma^{*k}_{ij}\partial_k$.
From Eq~\eqref{eq:local_duality}, it follows that $\nabla^*$ is dual to $\nabla$.
For uniqueness, from Eq.~\eqref{eq:local_duality} the connection comopnents $\Gamma^*_{kj,i}$ of $\nabla^*$ are uniquely determined given the metric $g_{ij}$ and the connection coefficients $\Gamma_{ki,j}$ of $\nabla$.
\end{proof}
\begin{definition}{(Dual $\alpha$-connections\cite{amari2000methods})}
Let $\nabla$ and $\nabla^*$ be two dual torsion-free connections, with respect to the metric $g$.
Consider the one-parameter family of connections given by the convex combination of the foregoing dual connections as
\begin{equation}
    \nabla^{(\alpha)} = \frac{1+\alpha}{2}\nabla^* + \frac{1-\alpha}{2}\nabla, \quad \alpha\in\mathbb{R}.
\end{equation}
$\nabla^{(\alpha)}$ is called the $\alpha$-connection, and play a central role in the study of statistical manifolds.
\end{definition}
\begin{proposition}
$\nabla^{(\alpha)}$ and $\nabla^{(-\alpha)}$ are dual connections with respect to the metric $g$.
\end{proposition}
\begin{proof}
From the duality of connections $\nabla$ and $\nabla^*$,
\begin{align}
    g(\nabla^{(\alpha)}_ZX, Y) &= \frac{1+\alpha}{2}g(\nabla^*_ZX, Y) + \frac{1-\alpha}{2}g(\nabla_ZX, Y) \nonumber \\
    &= \frac{1+\alpha}{2}\Big(Zg(X,Y) - g(X, \nabla_ZY)\Big) + \frac{1-\alpha}{2}g(\nabla_ZX, Y) \nonumber \\
    &= \frac{1+\alpha}{2}Zg(X,Y) - \frac{1+\alpha}{2}g(X, \nabla_ZY) + \frac{1-\alpha}{2}g(\nabla_ZX, Y).
\end{align}
Similarly,
\begin{equation}
    g(\nabla^{(-\alpha)}_ZY, X) = \frac{1-\alpha}{2}Zg(Y,X) - \frac{1-\alpha}{2}g(Y, \nabla_ZX) + \frac{1+\alpha}{2}g(\nabla_ZY, X).
\end{equation}
These relations yield
\begin{equation}
    g(\nabla^{(\alpha)}_ZX, Y) + g(\nabla^{(-\alpha)}_ZY, X) = Zg(X,Y),
\end{equation}
which shows that $\nabla^{(\alpha)}$ and $\nabla^{(-\alpha)}$ are dual connections.
\end{proof}

Now consider straight curves in the $\alpha$-coordinate systems using the following function.
\begin{definition}[$f$-interpolation~\citep{e23050528}]
\label{def:f_interpolation}
For any $a,b,\in\mathbb{R}$, some $\lambda\in[0,1]$ and some $\alpha\in\mathbb{R}$, $f$-interpolation of $a$ and $b$ is defined as
\begin{equation}
    m_f^{(\lambda,\alpha)}(a,b) \coloneqq c_\alpha f^{-1}_\alpha\Big\{(1-\lambda) f_{\alpha}(a) + \lambda f_{\alpha}(b) \Big\}, \label{eq:f_interpolation}
\end{equation}
where, $c_\alpha=2^{2/(1-\alpha)}$ and
\begin{equation}
    f_\alpha(a) \coloneqq \begin{cases}
    a^{\frac{1-\alpha}{2}} & (\alpha\neq 1) \\
    \log a & (\alpha = 1)
    \end{cases}
\end{equation}
is the function that defines the $f$-mean~\citep{hardy1952inequalities}.
\end{definition}
We can easily see that this family includes various known weighted means including the $e$-mixture and $m$-mixture for $\alpha=\pm 1$ in the literature of information geometry~\citep{amari2016information}:
\begin{align*}
&    m_f^{(\lambda,1)}(a,b) = (1-\lambda)\ln a + \lambda \ln b, & \;
     m_f^{(\lambda, -1)}(a,b) = (1-\lambda)a + \lambda b,\\ 
  & m_f^{(\lambda, 0)}(a,b) = \Big((1-\lambda)\sqrt{a} + \lambda\sqrt{b}\Big)^2, & \;
 m_f^{(\lambda, 3)}(a,b) = \frac{1}{(1-\lambda)\frac{1}{a} + \lambda\frac{1}{b}}. 
\end{align*}
Here, the normalization term is omitted for simplicity.
It has been shown that Eq.~\eqref{eq:f_interpolation} connects two probability distributions with a straight curve (geodesic) on the $\alpha$-coordinate system.

Then, one can consider the generalized weighted average as
\begin{equation}
    p(y|\bm{x};\bm{\theta}^*_\alpha) = f^{-1}_{\alpha}\Big\{\sum w_k f_\alpha(p(y|\bm{x};\bm{\theta}^*_k)) - C\Big\},
\end{equation}
where $C$ is a normalization term.

This is nothing but an $\alpha$-integration~\citep{amari2007integration} of the probability distributions, which is known to be optimal under $\alpha$-divergence $D_\alpha:\mathcal{P}\times\mathcal{P}\to\mathbb{R}_+$:
\begin{align}
\label{eq:mmix}
    \bm{\theta}^*_\alpha &= \argmin_{\bm{\theta}} R_\alpha(p(y|\bm{x};\bm{\theta})) \\
    R_\alpha(p(y|\bm{x};\bm{\theta})) &= \sum w_k D_\alpha\Big[p_{\bm{\theta}^*_k}\|p_{\bm{\theta}}\Big].
\end{align}
Using this, we can see that the parameters obtained by dropout based on $\alpha$-integration are included in the original $\mathcal{M}$ if and only if the model manifold is $\alpha$-flat.

By the way, from the results of existing theoretical analysis on dropout, it is known that dropout can be regarded as generating the geometric mean of a probability distribution~\citep{baldi2013understanding}.
This means that dropout can be viewed as manipulating a probability distribution in the $1$-coordinate system (or $e$-coordinate system), which is $\bm{\theta}\in\mathcal{M}$ if and only if the model manifold is $1$-flat, i.e.,
\begin{align}
    \bm{\theta}^*_1 \in \mathcal{M},
\end{align}
where
\begin{align}
    \bm{\theta}^*_1 = \argmin_{\bm{\theta}}\sum w_k D_{KL}\left[p_{\bm{\theta}^*_k}\|p_{\bm{\theta}}\right],
\end{align}
since $D_1 = D_{KL}$.
\begin{proposition}
The parameters estimated by dropout are included in the original model manifold $\mathcal{M}$ if and only if $\mathcal{M}$ is $1$-flat.
\end{proposition}

We note that the minimizer of the weighted average of the divergences of the form~\eqref{eq:mmix} is called $m$-mixture, while that with the inverse direction of divergence function $ \sum w_k D \Big[p_{\bm{\theta}}\| p_{\bm{\theta}^*_k}\Big]$ is called the $e$-mixture when $D=D_{KL}$ and the $u$-mixture when $D$ is in a class of Bregman divergence. The explicit form of those mixture is considered in~\citep{muratafujimoto2009} and \citep{HE_IG_2022}.

\begin{figure}[t]
    \centering
    \includegraphics[width=0.95\linewidth]{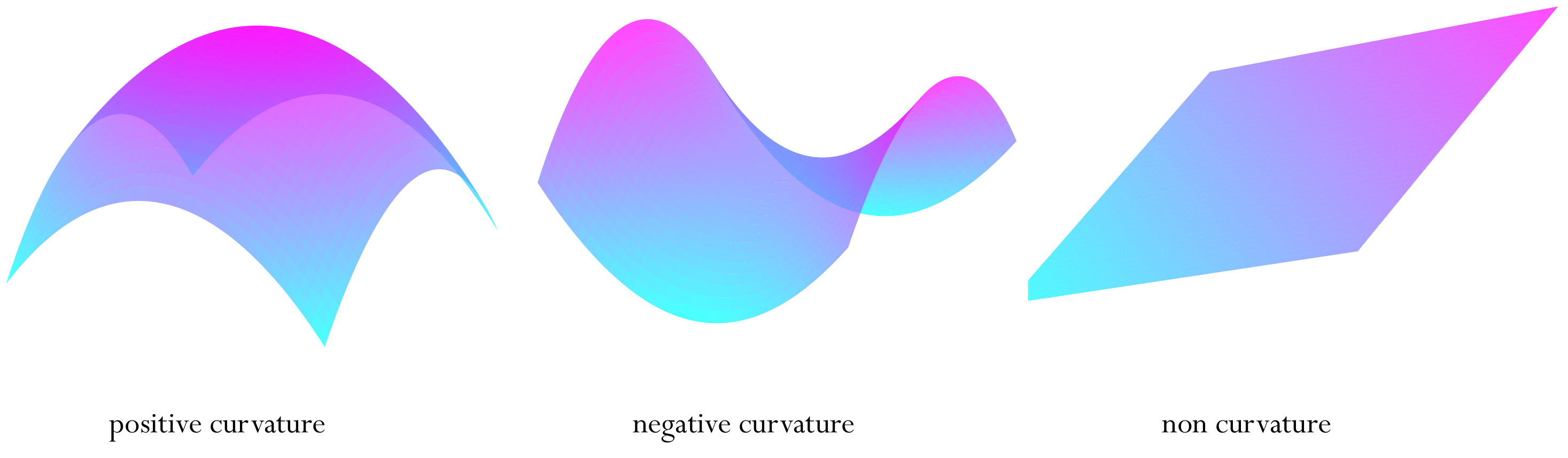}
    \caption{Curvature characterizes the deviation of local geometric properties of a manifold from the properties of the Euclidean geometry.}
    \label{fig:curvature}
\end{figure}

\subsection{When dropout outperforms ERM}
Using the Levi-Civita connections, the curvature tensor can be expressed in terms of the
metric tensor $g$ and its first and second partial derivatives.
Additionally, one can introduce the scalar curvature $R(p)$:
\begin{align}
    R(p) = g^{ij}(p)R^k_{kij}(p) = g^{ij}(p)g^{kl}(p)R_{kijl}(p)
\end{align}
for some point $p\in\mathcal{M}$.
The curvature tensor characterizes the deviation of local geometric properties of a manifold $\mathcal{M}$ from the properties of the Euclidean geometry (see Figure~\ref{fig:curvature}), and by using the scalar curvature, we can interpret geometric properties directly via the sign of $R(p)$ and volume elements.
That is, one can see intuitively that by writing the volume element ratio in terms of scalar curvature, the shape of the manifold changes when the sign of $R(p)$ is reversed.
For example, the volume of a small sphere about some point $p$ has smaller (larger) volume (area) than a sphere of the same radius defined on the $n$-dimensional real space $\mathbb{R}^n$ when the scalar curvature $R(p)$ is positive (negative) at that point.
Quantitatively, this behavior is described as follows:
\begin{align}
    \mathbb{V}(p|r) = \frac{\text{Vol}\left[\mathbb{S}^{(n-1)}(p|r)\subset\mathcal{M}\right]}{\text{Vol}\left[\mathbb{S}^{(n-1)}(p|r)\subset\mathbb{R}^n\right]} &= 1 - \frac{R(p)}{6(n + 2)}r^2 + O(r^4), \label{eq:volume_element}
\end{align}
where the notation $\mathbb{S}^{(m)}(p|r)$ represents a $m$-dimensional sphere with small radius $r$ centered at the point $p$.
For more details, see the textbooks of Riemannian geometry~\citep{do1992riemannian,petersen2006riemannian,jost2008riemannian}.

Geometrically, the conditions under which dropout training is superior to ordinary ERM are shown.
Let $\mathbb{V}(p|r)$ be the volume ratio of the original model manifold and $\mathbb{V}^*(p|r)$ be the flat manifold with the same point $p$.
Then, the difference is 
\begin{align}
    \mathbb{V}(p|r) - \mathbb{V}^*(p|r) = -\frac{R(p|r)}{6(n+2)}r^2,
\end{align}
and if the curvature is negative, $\mathbb{V}(p|r) - \mathbb{V}^*(p|r) > 0$.
Then, by mixing the parameters obtained from the dropout learning, we expect the model manifold is enlarged and it is able to obtain an estimator that is closer to the empirical distribution $q$.
Figure~\ref{fig:dropout_curvature} shows that when the curvature of model manifold $\mathcal{M}$ is negative, the dropout estimator could outperforms the ERM because it can be get closer to the empirical distribution, and when the curvature is positive, the dropout estimator is inferior to the ERM.

It is obvious that, for $1\leq k\leq n$, there exist nested submanifolds such that $\Theta(1)\subset\cdots\subset\Theta(k)\subset\cdots\subset\Theta$, where $\Theta(k)\subset\mathbb{R}^k$.
Similar to the discussion in Section~\ref{sec:flatness}, if $\mathcal{M}^D_k = \{p(y|\bm{x};\bm{\theta}_k)\ |\ \bm{\theta}_k\in\Theta(k)\subset\Theta\}$ is not flat, then the parameters generated by a dropout based on a smaller submanifold may not be included in $\mathcal{M}^D_k$.
That is, dropout estimator is included in $\mathcal{M}^D_k$ if and only if $\tilde{\nabla}$ is a  flat connection.
When $\mathcal{M}$ is not flat, it is better to construct the dropout estimator from as many $\Theta(k)$ as possible, which may allow us to obtain parameters that are not reachable in the original model manifold.

\begin{figure}[t]
    \centering
    \includegraphics[width=0.95\linewidth]{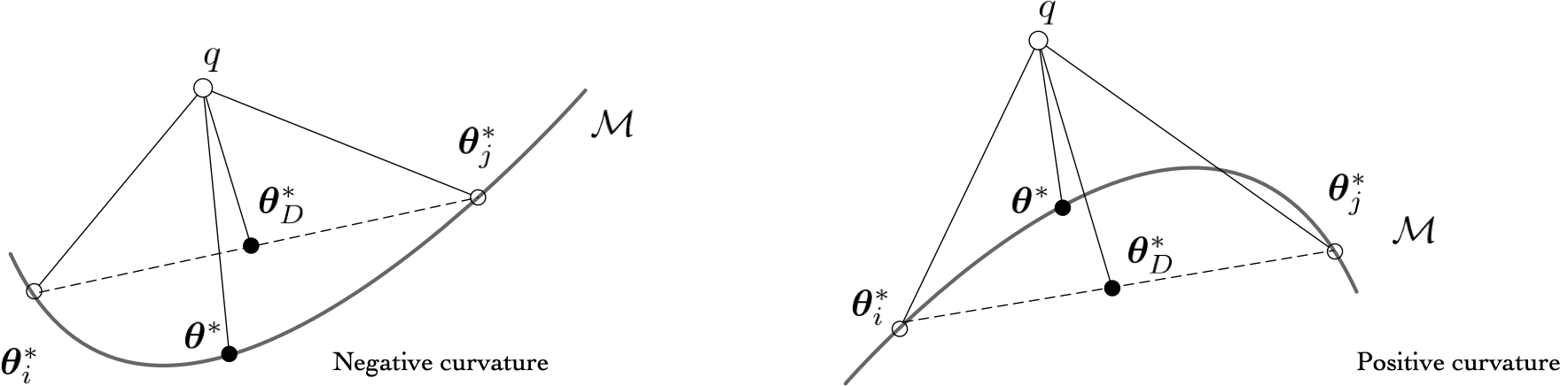}
    \caption{Relationship between the curvature of the model manifold $\mathcal{M}$ and the dropout estimator.
    When the curvature of $\mathcal{M}$ is negative, the dropout estimator outperforms the ERM, and when the curvature is positive, the dropout estimator is inferior to the ERM.}
    \label{fig:dropout_curvature}
\end{figure}

\section{Dropout submanifolds and Regularization}
\subsection{Geometry of dropout submanifolds}
Let $g$ be a metric of the model manifold $(\mathcal{M}, g, \nabla, \nabla^*)$.
Since $T_p\mathcal{M}^D_k\subset T_p\mathcal{M}$, the metric $\tilde{g}^k$ of $\mathcal{M}^D_k$ is naturally induced as
\begin{equation}
    \tilde{g}^k_p(\bm{v}, \bm{w}) \coloneqq g_p(\bm{v}, \bm{w}),\quad (\bm{v}, \bm{w}\in\mathcal{T}_p\mathcal{M}^D_k).
\end{equation}

In addition, each tangent space of $\mathcal{M}$ has the orthogonal decomposition
\begin{equation}
    T_p\mathcal{M} = T_p\mathcal{M}^D_k\oplus N_{p,k},
\end{equation}
where $N_{p,k} = \{Y \in T_p \mathcal{M}: g(Y,X) = 0,\ \forall X\in T_p\mathcal{M}^D_k\}$.
Then, any vector $Z\in T_p\mathcal{M}$ can be written as $Z=Z^T + Z^N$, with $Z^T\in T_p\mathcal{M}^D_k$ and $Z^N\in N_{p,k}$.
Applying the previous decomposition, for any $X,Y\in\mathcal{X}(\mathcal{M}^D_k)$ we have
\begin{align}
    \nabla_XY &= \tilde{\nabla}_XY + \bot(X,Y), \\
    \nabla^*_XY &= \tilde{\nabla}^*_XY + \bot^*(X,Y),
\end{align}
where
\begin{alignat}{2}
     \tilde{\nabla}_XY =& (\nabla_XY)^T, & \quad  \tilde{\nabla}^*_XY =& (\nabla^*_XY)^T \\
     \bot(X,Y) =& (\nabla_XY)^N, & \quad \bot^*(X,Y) =& (\nabla^*_XY)^N.
\end{alignat}
\begin{theorem}
The maps $\tilde{\nabla},\tilde{\nabla}^*:\mathcal{X}(\mathcal{M}^D_k)\times\mathcal{X}(\mathcal{M}^D_k)\to\mathcal{X}(\mathcal{M}^D_k)$ are torsion-free dual connections on $(\mathcal{M}^D_k, \tilde{g})$.
\end{theorem}
\begin{proof}
It suffices to check the $\mathcal{F}(\mathcal{M}^D_k)$-linearlity in the first argument and the Leibnitz' rule in the second as follows:
\begin{align*}
    \tilde{\nabla}_{fX}Y &= (\nabla_{fX}Y)^T = (f\nabla_XY)^T = f(\nabla_XY)^T = f\tilde{\nabla}_XY, \quad (\forall X,Y\in\mathcal{X}(\mathcal{M}^D_k)) \\
    \tilde{\nabla}_X(FY) &= (\nabla_XfY)^T = (X(f) + f\nabla_XY)^T = X(f) + f\tilde{\nabla}_XY, \quad (\forall X,Y\in\mathcal{X}(\mathcal{M}^D_k)).
\end{align*}
It implies that $\tilde{\nabla}$ is a linear connection on $\mathcal{M}^D_k$.
Since $\nabla$ is torsion-free and $[X,Y]$ is tangent to $\mathcal{M}^D_k$,
\begin{equation*}
    \tilde{\nabla}_XY - \tilde{\nabla}_YX = (\nabla_XY - \nabla_YX)^T = [X,Y]^T = [X,Y],
\end{equation*}
and $\tilde{\nabla}$ is torsion-free.
For any $X,Y,Z\in\mathcal{X}(\mathcal{M})$,
\begin{align*}
    Z\tilde{g}(X, Y) &= g(\nabla_ZX, Y) + g(X, \nabla^*_ZY) \\
    &= g((\nabla_ZX)^T, Y) + g(X, (\nabla^*_ZY)^T) \\
    &= g(\tilde{\nabla}_ZX, Y) + g(X, \tilde{\nabla}^*_ZY),
\end{align*}
which shows that $\tilde{\nabla}$ and $\tilde{\nabla}^*$ are dual connections on $(\mathcal{M}^D_k, \tilde{g})$.
\end{proof}

\subsection{Regularization term equivalent to dropout}
\label{sec:fim_regularization}
In the previous section, we see that the dropout can be regarded as flattening the model manifold.
In this section, we consider a regularization term that is equivalent to dropout.
The flatness of a manifold $\mathcal{M}$ can be expressed in the following second fundamental form for tangent vectors $U,V$ belonging to its tangent space $T_p\mathcal{M}$:
\begin{align}
    L(U,V) &= (\nabla_U V)^\perp, \label{eq:second_fundamental_form}
\end{align}
where
\begin{align}
    (\nabla_U V)_p &= (\nabla_U V)_p^\parallel + (\nabla_U V)_p^\perp
\end{align}
is called as Gauss formula.
$L(U, V)$ is symmetric and linear and can be written as
\begin{align}
    L(U,V) = L(\bm{\xi}_\alpha, \bm{\xi}_\beta)U^\alpha V^\beta = L_{\alpha,\beta} U^\alpha V^\beta.
\end{align}
Obviously, $L=0$ when the coefficients $L_{\alpha,\beta}$ vanish.
Here, we define the norm as follows
\begin{align}
    \|L\| = \max\Biggl\{\frac{\|L(U, U)\|}{\|U\|}\ |\ U \in T_p\mathcal{M} \Biggr\}.
\end{align}
From the symmetry of $L$, the eigenvalues of suitably normalized $L$ are
\begin{align}
    \max_{\|U\| = 1}|L^k(U,U)| = |\lambda_k|,
\end{align}
and
\begin{align}
    \|L\| = (\lambda_1^2,\dots,\lambda_m^2)^{\frac{1}{2}}
\end{align}
holds.
Thus, the regularization using the second fundamental form can be written using its eigenvalues as
\begin{align}
    \mathcal{L}(\bm{\theta}; \mu) = \ell(y, \varphi(\bm{x};\bm{\theta})) + \mu\|L\|, \label{eq:fisher_regularization}
\end{align}
where $\varphi$ is the neural network parameterized by $\bm{\theta}$.
Here, the Levi-Civita connection in tangent space is $\nabla_{\bm{\xi}_i}\bm{\xi}_j = \frac{\partial^2h}{\partial\bm{\theta}_i\partial\bm{\theta}_j}$, so an orthogonal decomposition of this gives
\begin{align}
    \frac{\partial^2h}{\partial\bm{\theta}_i\partial\bm{\theta}_j} = \Biggl(\frac{\partial^2\varphi}{\partial\bm{\theta}_i\partial\bm{\theta}_j}\Biggr)^\parallel + \Biggl(\frac{\partial^2\varphi}{\partial\bm{\theta}_i\partial\bm{\theta}_j}\Biggr)^\perp. \label{eq:decomposition_levi_civita}
\end{align}
From Eqs.~\eqref{eq:second_fundamental_form} and \eqref{eq:decomposition_levi_civita}, we have
\begin{align}
    L_{ij} = \Biggl(\frac{\partial^2\varphi}{\partial\bm{\theta}_i\partial\bm{\theta}_j}\Biggr)^\perp = \frac{\partial^2\varphi}{\partial\bm{\theta}_i\partial\bm{\theta}_j} - \Biggl(\frac{\partial^2\varphi}{\partial\bm{\theta}_i\partial\bm{\theta}_j}\Biggr)^\parallel. \label{eq:decomposition_second_fundamental_form}
\end{align}
The first term of Eq.~\eqref{eq:decomposition_second_fundamental_form} is the Fisher information matrix of the neural network.

\subsection{Connection to other regularizers}
Using Eq~\eqref{eq:fisher_regularization}, we can relate dropout to other regularizers.
First, we re-formulate Eq~\eqref{eq:fisher_regularization} with some function $\Phi$ that depends on the Fisher information matrix $I(\bm{\theta})$ as follows.
\begin{align}
    \mathcal{L}(\bm{\theta}; \mu) = \ell(y, \varphi(\bm{x};\bm{\theta})) + \Phi(I(\bm{\theta})). \label{eq:generalized_fisher_regularization}
\end{align}
Using the fact that the Fisher information matrix can be written in terms of KL divergence with small changes in parameters, i.e.,
\begin{align}
    D_{KL}[\bm{\theta}\|\bm{\theta}+d\bm{\theta}] = \delta^\top_{\bm{\theta}}I(\bm{\theta})\delta_{\bm{\theta}},
\end{align}
the following remarks can be derived.

\begin{remark}
Let $\Phi(I(\bm{\theta})) = D_{KL}[\bm{\theta}\|\bm{\theta}+d\bm{\theta}]$.
In this case, Eq.~\eqref{eq:generalized_fisher_regularization} is equivalent to knowledge distillation~\citep{hinton2015distilling,gou2021knowledge} with teacher model parametrized by $\bm{\theta}+d\bm{\theta}$.
\end{remark}

\begin{remark}
Let
\begin{align}
    \Phi(I(\bm{\theta})) = \frac{\lambda^2 C}{8n} + \log\frac{D_{KL}[\bm{\theta}\|\bm{\theta}+d\bm{\theta}] + \frac{1}{\epsilon}}{\lambda},
\end{align}
where $\lambda$ is some parameter, $C$ is a constant, $n$ is the sample size and $\epsilon$ is the precision parameter.
In this case, Eq.~\eqref{eq:generalized_fisher_regularization} is equivalent to PAC-Bayesian regularization~\citep{catoni2003pac}.
\end{remark}

\begin{figure}
    \centering
    \includegraphics[width=0.9\linewidth]{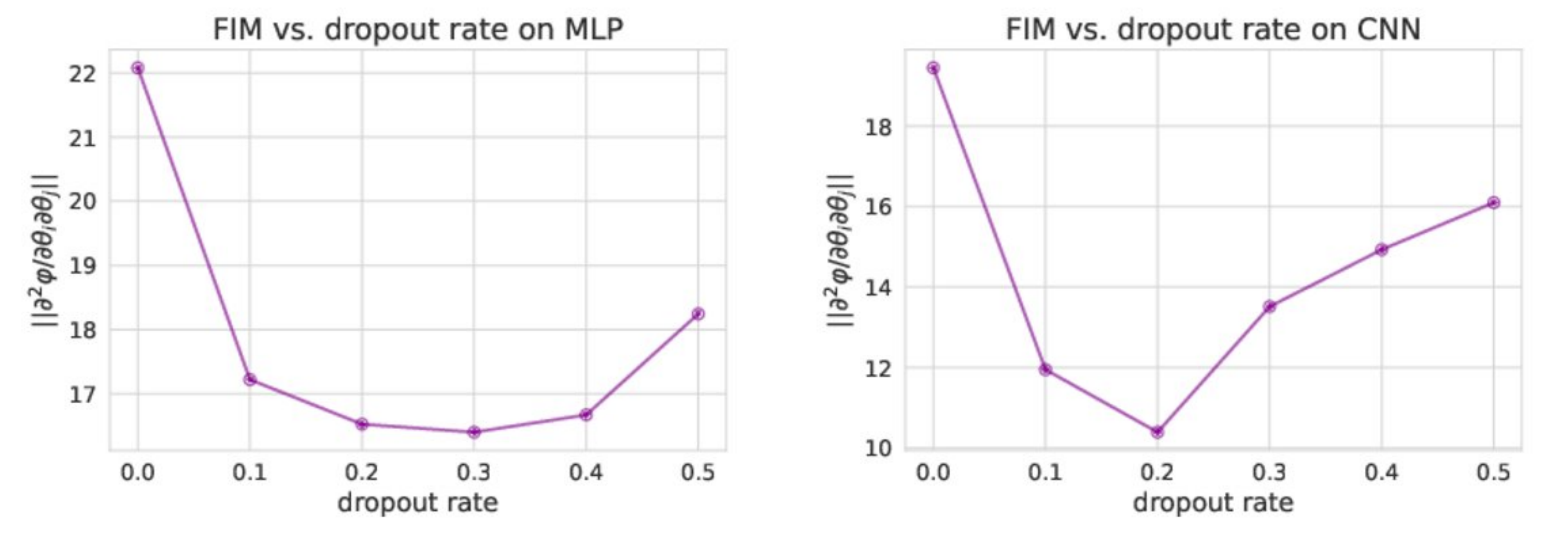}
    \caption{Norm of Fisher information matrix vs. dropout rate. We can see that applyig dropout reduces the norm of Fisher information matrix.}
    \label{fig:dropout_rate}
\end{figure}

\section{Numerical experiments}
In Section~\ref{sec:fim_regularization} we discussed the connection between dropout training and Fiser information matrix.
In this section, we confirm this relationship by the numerical experiments.

We use MNIST handwritten dataset~\citep{deng2012mnist} which has a training set of $60,000$ examples, and a test set of $10,000$ examples.
Each instance is $28\times 28$ gray-scale image.

Numerous studies have proposed approximation methods for the Fisher information matrix~\citep{heskes2000natural,martens2015optimizing,grosse2016kronecker,roux2007topmoumoute,ollivier2015riemannian}.
In our experiments, we use K-FAC~\cite{martens2015optimizing} method as the approximation of the Fisher information matrix.
This is one of the most well-known Fisher information matrix approximations and is sufficient to achieve the objective of numerically confirming the relationship between Fisher information matrix and dropout.
All experimental results are reported as the mean of $10$ trials.

Figure~\ref{fig:dropout_rate} shows the relationship between Fisher information matrix and dropout rate.
From this figure, we can see that applying dropout reduces the norm of Fisher information matrix.
In the previous study~\citep{baldi2013understanding}, it is suggested that the optimal parameter for dropout is $p=0.5$, but the numerical experiments suggest that $p=0.2$ or $p=0.3$ lead the minimum norm of FIM.
This results suggest that the second term of Eq.~\eqref{eq:decomposition_second_fundamental_form} is dominant with high dropout rate.

\section{Conclusion and discussion}
This study formulates dropout through the lens of information geometry.
Dropout is nothing but a perturbation of neurons in neural networks, and a deep understanding of this framework is expected to help us understand the behavior of neural networks.

\subsection{Future works}
\begin{itemize}
    \item Information geometry is a very powerful framework that has led to a number of new algorithms~\citep{amari1998natural,murata2004information,akaho2004pca,zhang2019cooperative}. It is very important future research to derive new dropout-inspired algorithms based on insights into the algorithms from a geometric perspective.
    \item The information geometry used in this paper is a mathematical tool based on the Fisher metric.
    Apart from this, one of the most popular geometric frameworks in the field of machine learning today is Wasserstein geometry~\citep{villani2008optimal}.
    More recently, a novel work unifying Fisher information geometry and Wasserstein geometry has been published~\citep{amari2018information}.
    In addition to the analysis of the algorithm in Fisher information geometry, it is expected that an analysis based on optimal transport will allow for a deeper understanding of the algorithm.
    \item we showed that dropout essentially corresponds to a regularization that depends on the Fisher information, and supported this result from numerical experiments. This result suggests that dropout and other regularization methods can be generalized via the Fisher information matrix.
\end{itemize}

\input{output.bbl}

\end{document}

%% file: output.bbl
 \newcommand{\noop}[1]{}